\documentclass[10pt,twocolumn,letterpaper]{article}

\usepackage{cvpr}              
\usepackage{comment} 
\usepackage[dvipsnames, table]{xcolor}

\usepackage{multirow}
\usepackage{array}
\usepackage{acronym}
\usepackage{wrapfig}
\usepackage{tikz}
\usepackage[ruled,vlined]{algorithm2e}

\usepackage{mathtools}

\definecolor{MethodColor}{gray}{0.4}
\definecolor{CellBck}{gray}{0.92}

\newcommand{\cbg}{\cellcolor{CellBck}}

\acrodef{FA}[FA]{Frame Averaging}
\acrodef{MC}[MC]{Monte Carlo}
\acrodef{kNN}[\emph{k}NN]{\emph{k} nearest neighbors}
\acrodef{MLP}[MLP]{Multi-Layer Perceptron}
\acrodef{PCA}[PCA]{Principal Component Analysis}
\acrodef{LRF}[LRF]{Local Reference Frame}
\acrodef{FPS}[FPS]{frames per second}

\newcommand{\method}[1]{{#1}}

\usepackage{amsthm}
\newtheorem{theorem}{Theorem}

\definecolor{cvprblue}{rgb}{0.21,0.49,0.74}
\usepackage[pagebackref,breaklinks,colorlinks,citecolor=cvprblue]{hyperref}
\usepackage{gensymb}

\def\paperID{417} 
\def\confName{3DV\xspace}
\def\confYear{2025\xspace}

\title{Efficient Continuous Group Convolutions\\for Local SE(3) Equivariance in 3D Point Clouds}

\author{Lisa Weijler\\
TU Wien, Austria
\and
Pedro Hermosilla\\
TU Wien, Austria}

\begin{document}
\maketitle
\begin{abstract}
  Extending the translation equivariance property of convolutional neural networks to larger symmetry groups has been shown to reduce sample complexity and enable more discriminative feature learning. Further, exploiting additional symmetries facilitates greater weight sharing than standard convolutions, leading to an enhanced network expressivity without an increase in parameter count. However, extending the equivariant properties of a convolution layer comes at a computational cost. In particular, for 3D data, expanding equivariance to the SE(3) group (rotation and translation) results in a 6D convolution operation, which is not tractable for larger data samples such as 3D scene scans. While efforts have been made to develop efficient SE(3) equivariant networks, existing approaches rely on discretization or only introduce global rotation equivariance. This limits their applicability to point clouds representing a scene composed of multiple objects. This work presents an efficient, continuous, and local SE(3) equivariant convolution layer for point cloud processing based on general group convolution and local reference frames. Our experiments show that our approach achieves competitive or superior performance across a range of datasets and tasks, including object classification and semantic segmentation, with negligible computational overhead. The code for our implementation is available at \href{https://github.com/lisaweijler/SE3Conv3D}{this repository}.

\end{abstract}    
\section{Introduction}
\label{sec:intro}

\begin{figure}[t]
  \centering
    \includegraphics[width=\linewidth]{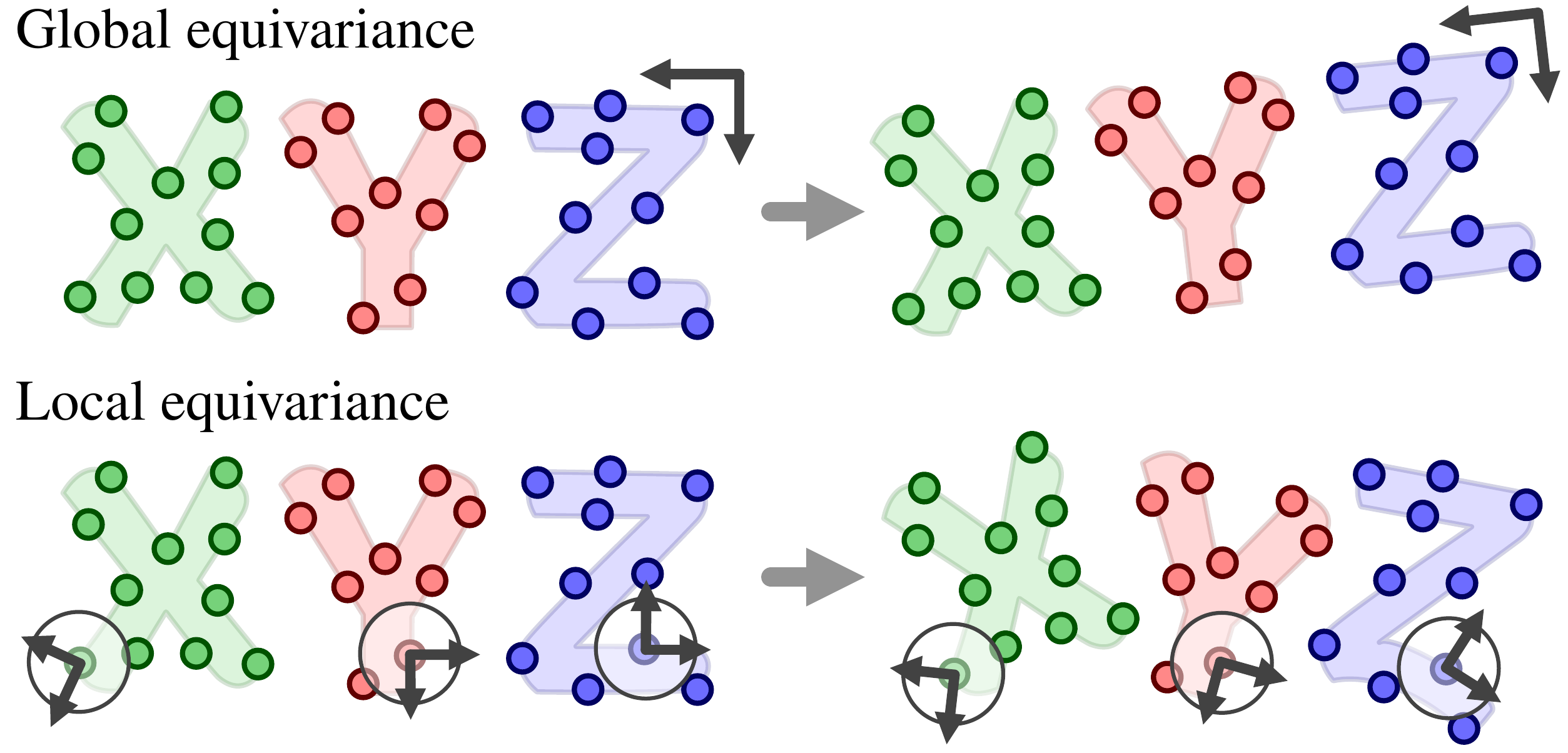}

  \captionof{figure}{
  While \textbf{global equivariant} designs ensure robustness to whole-scene rotations, they fail with randomly rotated scene parts or elements. In contrast, \textbf{local equivariant} operations maintain robustness by handling local geometry rotations around each point.
  }
  \label{fig:teaser}
\end{figure}

In 3D vision, point clouds are the most commonly used representation to process 3D data given that they are relatively cheap to capture and process. This representation is composed of multiple point coordinates, with additional attributes such as color or normal vector, from samples on the surface of 3D objects.
In the past years, several neural network architectures have been proposed to process such data~\cite{qi2017pointnet,atzmon2018pccnn,thomas2019KPConv,boulch2020convpoint,hermosilla2018mccnn,wu2019pointconv}.
Approaches learning directly from 3D data often take inspiration from the success in 2D vision and address two of the main challenges in such data representation, order invariance and translation equivariance. Yet, 3D data entails more variations and complex group transformations due to an increased number of degrees of freedom (DoF); for the roto-translation group SE(3), $\text{DoF}=6$. Objects in 3D space do not have a predefined canonical orientation and many rotational variances are present. 

Equivariance is the property of an operator that allows the prediction of the transformation of the output given an input transformation, while group-invariant operators produce identical features under various group transforms of the input. The latter can be seen as an information loss; they struggle to differentiate between unique instances with internal symmetries, e.g., "$8$" vs. "$\infty$". Baking SE(3)-equivariance into the network architecture can thus be beneficial since equivariant features maintain information about the input group transform across neural layers, making them more expressive and generalizable by capturing the variance that is present in the data. 

Traditionally, to obtain such properties, data augmentation techniques are used, but this requires neural networks to store latent orientations of the objects, limiting the network capacity.
Whilst this might be a viable solution for network architectures for 2D images, neural networks for 3D point clouds usually require large amounts of memory, limiting the number of parameters of the models and making it impossible to achieve such equivariance from the data.
Recent advances have been made to address this problem~\cite{deng2021vector, puny2022frame}, where several neural network architectures can match or even surpass the performance of the standard architectures relying on data augmentations.

Unfortunately, many of those solutions only address the problem of global rotation equivariance, \ie rotations of a single 3D object or scene as a whole.
3D objects or scenes are composed of multiple parts or objects that can have arbitrary orientations \wrt each other, see Fig.~\ref{fig:teaser}.
The relative orientations of different objects in the scene cannot be captured by global equivariance as obtained by existing architectures or by data augmentation techniques.

Group convolution is an operation that is, per definition, equivariant to a specific group and, hence, capable of coping with such problems.
These operations aggregate information from neighboring samples, not only from the translation group T(3) as standard convolutions, but from the rotation and translation group SE(3).
By restricting the receptive field of these operations, they become rotation equivariant \wrt the local geometry inside the receptive field, allowing them to be equivariant to the relative rotations of different parts of the scene.
To successfully compute such operations in the continuous domain a complex integral over the full group needs to be solved (6D convolution), which makes such operations not practical for large networks since it has large memory and computational burden. Further, defining a grid on SE(3) is not trivial, where recent works try to address these problems by using \ac{MC} integration~\cite{finzi2020generalizing} or by discretizing the SO(3) group~\cite{chen2021equivariant, zhu2023e2pn}.
However, as we will show, these approximations limit the performance of the network.

In this paper, we propose using a finite subset $\mathcal{F}(x) \subset$ SE(3), referred to as a frame, to solve the group equivariant integral, which allows for exact equivariance (as opposed to approaches based on \ac{MC} sampling or discretization), while reducing the computational burden. 
The elements $g \in \mathcal{F}(x)$ can be seen as \ac{LRF}, that together with their corresponding point $x$ build a grid on SE(3), where the integral of the group convolution can be computed efficiently.
Further, our approach stochastically samples $g \in \mathcal{F}(x)$ during training with only a few samples, two or even one, which reduces the additional computational and memory burden significantly and for the case of one sample to almost zero.
Our extensive experiments show that such group convolution is able to achieve local rotation equivariance, surpassing other local equivariant designs by a large margin.
Moreover, our experiments also show that a network constructed using such convolutions as building blocks is able to be robust to local transformations not seen during training, where popular global equivariant frameworks fail.

\section{Related work}
\label{sec:related_work}
This section gives an overview of networks that can process unstructured data, such as point clouds, with a focus on specific point architectures that are equivariant to rotations.

\textbf{Point-based neural networks.}
The first neural network architecture specifically designed to process point clouds was PointNet~\cite{qi2017pointnet}. The idea of directly processing point clouds was followed by several works incorporating concepts and designs from convolutional neural networks for images into the continuous domain.
Atzmon~\etal~\cite{atzmon2018pccnn}, Thomas~\etal~\cite{thomas2019KPConv}, and Boulch~\etal~\cite{boulch2020convpoint} propose a convolution operation based on a set of points located inside a receptive field and a correlation function as the kernel function.
Another line of research, including Hermosilla~\etal~\cite{hermosilla2018mccnn} and Wu~\etal~\cite{wu2019pointconv} uses a convolution operation with a kernel function represented by an \ac{MLP}, that takes the relative position between points as input. 
In this work, similarly, we use an \ac{MLP} as our kernel, but with the relative orientation in addition to the relative position between points as input.

\textbf{Rotation equivariant point networks.}
Equivariance or invariance to SE(3) can be achieved by modifying the model's input through data augmentation or by adapting the internal operations of the model to have such equivariance by construction. 
Several works have designed equivariant network architectures by aligning the input point cloud to a reference frame before being processed by the model.
Gojcic~\etal~\cite{gojcic20193DSmoothNet} align local patches of a point cloud to their \ac{LRF} defined by the normal to solve the task of keypoint matching.
Xiao~\etal~\cite{xiao2020pca} uses \ac{PCA} to build a frame to transform the input point cloud and an attention network to aggregate features over the different transformations.
Other works, instead, adopt a different approach, in which they compute invariant local features and use these as input to a convolution operation.
Zhang~\etal~\cite{zhang-riconv-3dv19} use angles and distances as input to a local PointNet architecture to achieve rotation invariance.
Later, Zhang~\etal~\cite{zhang2022riconv2} extended this work with additional local features.
Yu~\etal~\cite{yupositionalfeature2020} also used distances and angles to achieve global equivariance.
However, none of these approaches achieved the goal of our work, local SE(3) equivariance.
Recently, Puny~\etal~\cite{puny2022frame} suggested a general framework to achieve equivariance on any neural network by averaging the model's output over a subset of the group elements. Concurrently with our work, Atzmon~\etal~\cite{atzmonapproximately} introduced a piecewise E(3) equivariant approach applying ~\cite{puny2022frame} to multiple parts proposed by a partition prediction model, where the locality is determined by the partition of the objects.
In this work we use the concepts of~\cite{puny2022frame} to achieve local SE(3) equivariance efficiently without the need for a partition prediction.

Another line of work achieves local equivariance by making the model's internal features \emph{steerable}~\cite{cohen2016steerable}, \,  i.e., the feature values transform predictably as the input transforms.
In 3D, these works usually rely on the theory of spherical harmonics to obtain \emph{steerable} features~\cite{thomas2018tensor, fuchs2020se3transformers, weiler20183dsteercnn}. 
Vector Neurons~\cite{deng2021vector} also uses higher-order features, representing each feature as a 3D vector, to achieve global SO(3) equivariance.
Unfortunately, this increases the memory consumption of the models and restricts the kernel representation used.

More related to our work is the concept of group convolutions~\cite{cohen2016group}.
These operations generalize the concept of convolutions and extend the equivariance to the translation group of standard convolution to any group.
These ideas have been applied to the SE(3) group for voxelized representations, where the group has a finite number of elements~\cite{worrall2018cubenet}, and to point clouds in the continuous domain by discretizing the continuous SE(3) group using the icosahedral group~\cite{chen2021equivariant, zhu2023e2pn}.
Although the computation of such group convolution can be implemented with permutation matrices, a large number of group elements requires a significant computational burden.
To address this issue, Chen~\etal~\cite{chen2021equivariant} proposed a separable convolution allowing for fast computation of the group convolution.
Zhu~\etal~\cite{zhu2023e2pn} instead proposes to use the SO(2) group as the stabilizer subgroup to form spherical quotient feature fields.
Unfortunately, these discretizations require lifting the feature representation to the size of the discrete group used, increasing the memory requirement of the model by a factor equal to the size of the discrete group.
Recent works have suggested solving the group convolution integral by using \ac{MC} sampling on the continuous group~\cite{finzi2020generalizing, Hutchinson2021lietransformer}.
These works randomly sample the group and use farthest point sampling to select a subset from which the integral is approximated with \ac{MC} integration.
However, this approach might require large samples to obtain a reasonable estimation of the integral and hence suffer from substantial memory load.

In this work, we also suggest using group convolutions on the continuous domain to achieve local rotation equivariance.
However, our sampling strategy allows us to solve this integral only with a few samples on the SO(3) group, rendering group convolutions a viable solution to achieve equivariance on standard deep point-based architectures with negligible computational or memory requirements.
\section{Methods}
\label{sec:methods}

In this section, we describe our proposed approach.
First, the reader is introduced to the concept of group equivariant convolutions. Then, our efficient continuous group convolution is described in detail.

\subsection{Group equivariant convolution}

An intuitive way of thinking about convolutions is the notion of template matching, where a kernel $k$ is 
shifted over a feature map $f$ to detect patterns. In the continuous case, we consider a feature map $f: X \rightarrow \mathbb{R}^c$ as a multi-channel scalar field and $\mathcal{X} = (\text{L}^2(X))^c$ as the space of feature maps over some space $X$. A more formal definition of a convolution layer is then given as a learnable kernel operator $\Phi: \mathcal{X}\rightarrow \mathcal{Y}$ that transforms feature maps $f$ as follows
\begin{equation}
[\Phi f](y) = (f \star k)(y) = \int_{X} f(x) k(x-y) d x,
\end{equation}
with $X=Y=\mathbb{R}^d$, where $d=3$ for point clouds. (Note that the definition given is cross-correlation instead of convolution since this aligns better with template-matching.) It is well known that convolution layers are translation equivariant due to the shifted kernel, i.e., the kernel is only dependent on relative distances: if the input feature map is shifted, the output feature map follows the same transformation. Yet, since relative distances hold directional information that changes under rotations, it is self-evident that a convolution layer is not equivariant to rotations. One solution is to use $\| x- y \|$ as input to the kernel at the cost of losing the capacity to capture directional features.

We say that an operator $\Phi$ is equivariant to a specific Group $G$ if it commutes with group representations on the input and output feature maps, meaning $\forall g \in  G: \rho^\mathcal{Y}(g)\circ\Phi = \Phi \circ \rho^\mathcal{X}(g)$, where $\rho^\mathcal{X}(g)$ is the regular group representation of $g$ that transforms a function $f\in \mathcal{X}$ by shifting its domain via $g^{-1}$. If the output feature map is left unaltered, $\Phi$ is $G$-invariant. 
Various important works in the field of equivariant deep learning~\cite{cohen2019general,bekkers2020bspline,kondor2018generalization} show or conclude that a linear operator $\Phi$ that maps between feature maps on homogeneous spaces $X$, $Y$ of a group $G$, is $G$-equivariant iff it is a kernel operator (also often called integral operator) with a single-valued kernel (only dependent on relative values). Further, considering $Y=G/H$ as quotient space with $H=\{g\in G|g y_0 = y_0\}$ as the stabilizer subgroup $\text{Stab}_G(y_0)$, which consists of group elements that leave a chosen origin $y_0 \in Y$ unchanged, the kernel of a $G$-equivariant $\Phi$ must be invariant towards elements of $H$ (invariance constraint). 

When looking at the concrete example $Y=\mathbb{R}^d$, $G=\text{SE}(d)$, we say $\mathbb{R}^d \equiv \text{SE}(d)/\text{SO}(d)$ is a quotient space with stabilizer subgroup SO$(d)$; an intuition is given in the following.
Since $\mathbb{R}^d$ is a homogeneous space of SE$(d)$, every point $x\in \mathbb{R}^d$ can be reached from the origin $\mathbf{0}\in \mathbb{R}^d$ by a group element, a roto-translation, $(\text{t, R}) \in \text{SE}(d)$. 
In fact there exist several group elements such that $x = (\text{t, R})\mathbf{0} = \text{R}\mathbf{0}+\text{t}$, namely any group element with $\text{t}=x$ regardless of the rotation part as any rotation $\text{R}\in \text{SO}(d)$ leaves the origin unchanged. 
Hence if $Y=\mathbb{R}^d$ and $G=\text{SE}(d)$, the kernel of an $\text{SE}(d)$-equivariant $\Phi$ must be $\text{SO}(d)$-invariant, meaning one could only use isotropic kernels, which severely limits the expressivity of patterns that can be detected e.g. using  $\| x- y \|$ as discussed above. In order to not limit the representation power of the kernel while achieving SE$(d)$ equivariance, the feature maps need to be lifted to 

the group itself $Y=G$ since then the stabilizer subgroup only consists of the trivial element $ H =\{\mathbf{e}\}$, and the kernel is no longer constrained. Note that $Y$ and $X$ do not necessarily have to be the same space.
Consequently, to extend the translation equivariance of convolution layers to arbitrary affine Lie groups 

three types of layers can be used~\cite{bekkers2020bspline}:
\begin{itemize}
    \item \textbf{Lifting layer} ($X=\mathbb{R}^d$, $Y=G$, $H=\{\mathbf{e}\}$): 
    \begin{equation}
    (f \star k)(g) = \int_{R^d} f(x) k(g^{-1}x) dx
    \end{equation}
    For $G=$SE$(d)$, this can be viewed as a template matching various rotated versions of the kernel, creating a feature map for different positions and rotations.
    \item \textbf{Group convolution layer} ($X=G, Y=G, H=\{\mathbf{e}\})$:
    \begin{equation}
    (f \star k)(g) = \int_{G} f(g') k(g^{-1}g') d\mu(g')
    \label{eq:Gconv}
    \end{equation}
    This layer constitutes the convolution on the full group, e.g., it conducts template matching over all possible combinations of positions and rotations from the input and output feature map.
    \item \textbf{Projection layer} ($X=G$, $Y=\mathbf{R}^d$, $H=\text{Stab}_G(\mathbf{0})$):
    \begin{equation}
    (f \star k)(x) = \int_{H} f(x,h')  d\mu(h')
    \label{eq:proj-layer}
    \end{equation}

    For tasks like point-wise classification, the final prediction must be invariant, so feature maps or rotations are projected to their corresponding point in $\mathbb{R}^d$. This layer is omitted for tasks like pose estimation.
\end{itemize}

\subsection{Efficient group convolution}
Since group convolution layers map between higher dimensional feature maps and must compute the integral over the entire group, they can introduce a computational bottleneck. In the case of 3D point clouds and the affine group $\text{SE(3)} = \mathbb{R}^3 \rtimes \text{SO(3)}$, ~\cref{eq:Gconv} turns into a 6D convolution $(f \star k)(g)$ with $g=(\text{x, R}) \in \text{SE(3)}$, which can be written as a double integral
\begin{equation}
\int_{\mathbb{R}^3} \int_{\text{SO(3)}} f(\text{t, R'})k(\text{R}^{-1}(\text{t} - \text{x}), \text{R}^{-1}\text{R'}) d\text{t} d\mu(\text{R'}),
\label{eq:double-int}
\end{equation}
with $\mu(\cdot)$ being the Haar measure on SO$(3)$.

In addition to the computational burden of a 6D convolution, another difficulty lies in how to define a grid on SE$(3)$ or, more specifically, on the SO$(3)$ part to compute the integral of ~\cref{eq:double-int}. 
Previous works such as \cite{zhu2023e2pn, chen2021equivariant} have relied on the discretization of SO(3) using platonic solids that assign to each spatial component the same finite grid on SO(3) to make it tractable, yet at the loss of continuity and exact equivariance.
To stay in the continuous domain, similarly to the work of Finzi~\etal~\cite{finzi2020generalizing}, one can use \ac{MC} approximation for both the spatial and rotational part to solve the double integral
\begin{equation}
    \sum_{j} \frac{1}{\lvert H'_j \rvert}\sum_{(\text{t, R'})\in H'_j}  f(\text{t, R'})k(\text{R}^{-1}(\text{t} - \text{x}), \text{R}^{-1}\text{R'}),
    \label{eq:lie-conv}
\end{equation}
\noindent where $j$ are the indices of the points $x_j \in \mathbb{R}^3$ of the point cloud
and $H'_j = \{(\text{t, R'}) | \text{t}=x_j, \text{R'}\in \text{SO(3)}\}$ is the set of SE(3) group elements that result form lifting points $x_j$ to SE(3) by repeating the point coordinate with uniformly sampled rotations. Note that the point cloud is treated as a sparse feature map that defines the sampling of the spatial component.

 Using \ac{MC} approximation can be thought of as defining a random grid on SE(3). Hence, the approximation quality of this integral depends on the number of sampled group elements or, more precisely, on the number of rotations $\lvert H'_j \rvert = O$ sampled per point $x_j$; the approximation error converges towards zero for $O \rightarrow \infty$.
However, sampling $O$ rotations per point increases the model's memory by a factor of $O$.
Moreover, the required computations for the convolution also increase by a factor of $O^2$. Hence, using \ac{MC} results in a trade-off between computational efficiency and preciseness of equivariance property, showing that an efficient grid on SE(3) that allows for exact equivariance with finite rotation elements is crucial to make continuous group convolutions practical for point-based networks.

\textbf{Efficient grid on SE(3).}
To achieve exact equivariance with tractable computational load, we propose a carefully constructed grid $\mathcal{F}(x_j) \subset SE(3)$ specific to each point $x_j \in \mathbb{R}^3$.
Note that while $H_j$ in \cref{eq:lie-conv} was also dependent on $x_j$, the grid was still the same for each point, namely the entire group, where the dependency merely came from approximation by sampling.

We call $\mathcal{F}(x):\mathbb{R}^3 \rightarrow 2^{\text{SE}(3)}$ a frame, which is a set-valued function and maps a point in space to a set of group elements such that $\forall (\text{t, R}) \in \mathcal{F}(x): x = \text{t}$. A frame is called $G$-equivariant if $\forall g \in G: g\mathcal{F}(x) = \mathcal{F}(gx)$. Using $\mathcal{F}(x)$ as grid, we define a 3D sparse point cloud group convolution layer $\Phi_{\mathcal{F}}$ as

\begin{equation}
    \sum_{j} \frac{1}{\lvert \mathcal{F}(x_j) \rvert}\sum_{(\text{t, R'})\in \mathcal{F}(x_j)}  f(\text{t, R'})k(\text{R}^{-1}(\text{t} - \text{x}), \text{R}^{-1}\text{R'}).
    \label{eq:loco-roto-conv}
\end{equation}
$\Phi_{\mathcal{F}}$ thus transforms feature maps $f: X \rightarrow \mathbb{R}^c$, defined on the domain $X = \{\mathcal{F}(x)|x\in \mathbb{R}^3\}$.
Using those definitions, we can formulate the following.
\begin{theorem}
    Let $\mathcal{F}$ be an SE(3)-equivariant frame. Then, $\Phi_{\mathcal{F}}$ is SE(3)-equivariant. 
    \label{theorem}
\end{theorem}
\begin{proof}
    See suppl. mat.
\end{proof}

Since $\mathcal{F}(x)$ can be constructed with local PCA, as explained below, it only consists of a few elements and the amount of computations is significantly reduced.

\begin{figure}[t]
    \centering
    \includegraphics*[width=\linewidth]{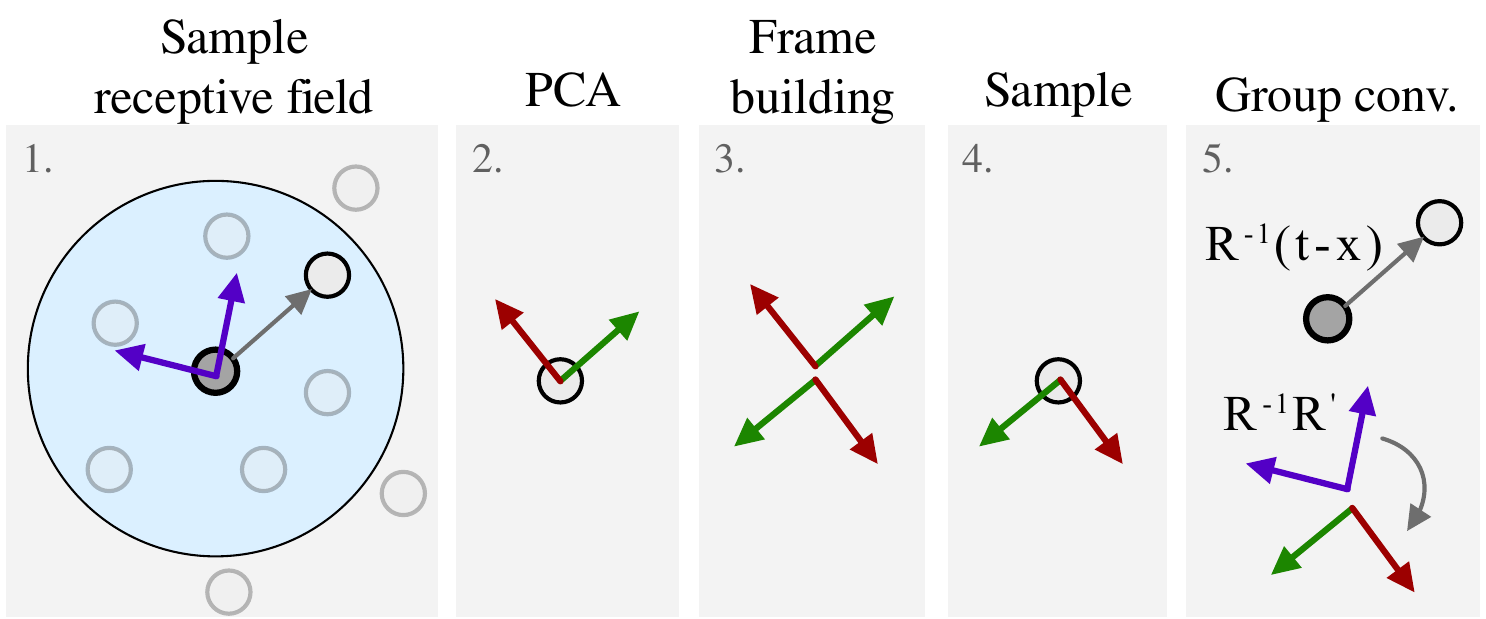}
    \caption{
    Overview of our convolution operation. 
    Given a central point with an orientation, 
    first, we sample neighboring points. 
    For each point, we use PCA to build a frame from it.
    Then, we sample an orientation from the frame.
    Then, the input to the group convolution kernel is the relative position plus the relative orientations between points.}
    \label{fig:main_idea}
\end{figure}

\textbf{Frame Construction.}
We compute PCA over a region around the point to construct $\mathcal{F}(x)$.
Due to the ambiguity of PCA \wrt the direction of the different axes, we follow Xiao \etal~\cite{xiao2020pca} and Puny \etal~\cite{puny2022frame} and construct 4 different \ac{LRF} by inverting the sign of the different directions.
Given the eigenvectors $[v_1, v_2, v_3]$ of the covariance matrix $C$ of the point coordinates, the frame can be defined as $\mathcal{F}(x) = \{ ( [ \alpha_1 v_1, \alpha_2 v_2, \alpha_3 v_3], t ) | \alpha_i \in \{1, -1\}\}$.
To be equivariant to the SE(3) group, $\mathcal{F}(x)$ is restricted to orthogonal, positive rotation matrices.
This results in a frame with a finite number of elements, $| \mathcal{F} (X) | = 2^{3-1} = 4$.

\textbf{Stochastic Approximation.}
Although $\mathcal{F}(x)$ only has 4 elements, this might still be restrictive for modern state-of-the-art deep architectures used to process large 3D scenes.
Therefore, we propose to perform a stochastic approximation of~\cref{eq:loco-roto-conv} during training by only sampling a subset of the elements of $\mathcal{F}(x)$ for input and output domains of the feature maps.
In particular, we propose randomly sampling two or even only one element of $\mathcal{F}(x)$ for each point $x$ in the point cloud where the convolution will be computed.
Then, during the computation of ~\cref{eq:loco-roto-conv}, only the sampled elements for points $x_j$ are used to approximate the SO$(3)$ integral.
Our approach is illustrated in~\cref{fig:main_idea}.

While using all elements of $\mathcal{F}(x)$ would increase the memory consumption of a standard model by a factor of $4$ and the number of computations by a factor of $16$, sampling $2$ elements would only increase the memory by a factor of $2$ and the computations by a factor of $4$.
More importantly, randomly sampling only $1$ element will maintain the memory consumption and computations equal to the model with standard convolutions.
During testing, since large batches are not necessary, we can use the full frame $\mathcal{F}(x)$ to compute ~\cref{eq:loco-roto-conv}. The error introduced by stochastic approximation by subsampling 2 or 1 element instead of using all 4 is discussed in the supplementary materials.

\textbf{Local vs Global Equivariance.} 
In practice, the locality of the kernel is enforced by calculating the convolution for a local neighborhood $N_x=\{x_j \in \mathbb{R}^3|\|x_j - x\| <r\}$ of $x$ only. 
Equivariance of~\cref{eq:lie-conv} and ~\cref{eq:loco-roto-conv} is ensured on a scale that depends on the receptive field used. 
Since we only consider a small receptive field around each point, our operations become equivariant \wrt rotations of the local geometry within this receptive field. By incorporating several layers in our architectures with increasing receptive fields, the model is able to capture patterns at different scales in an equivariant manner.
Ultimately, the whole model also covers the global equivariance scale since the last layers have an effective receptive field covering the entire scene.

\section{Experiments}
\label{sec:experiments}

We conduct experiments on object classification, and semantic segmentation to validate our methods. Due to space constraints, additional experiments, ablation studies, detailed dataset description and implementation are provided in the supplementary materials.

\subsection{Baselines.}
In our main experiments, we compare our convolution operation, \method{Ours}, to the same model where the integral is solved using \ac{MC}~\cite{finzi2020generalizing}, \method{MC}, and a model using standard convolutions, \method{STD}.
Moreover, we also compare to additional rotation equivariant networks, relying on global and local equivariant designs.

\subsection{Shape classification}
We use the task of shape classification to measure the equivariant capabilities of the models \wrt global rotations.
For this task, predictions must be invariant of the rotation applied to the model.
We use a global pooling operation as the projection layer (~\cref{eq:proj-layer}) at the end of our encoder to transform the equivariant features into invariant ones.

\textbf{Dataset.}
We use the ModelNet40 dataset~\cite{wu2015modelnet} since this is a standard benchmark for rotation equivariant networks~\cite{deng2021vector}.
Our model only takes as input point coordinates, and performance is measured with overall accuracy.

\textbf{Experimental setup.}
We provide different configuration setups in our experiments.
All models are evaluated when trained and tested without any rotation, I / I.
Further, we evaluate all models trained without any rotation but random rotations during testing, I / SO(3).
Lastly, we evaluate our models with random rotations during training and testing.
Additionally, to compare to other state-of-the-art methods, we take the commonly used setup where random rotations are applied along the up vector during training and random rotations on SO(3) during testing, z / SO(3).
Although this setup is less challenging than I / SO(3), it allows us to compare to additional rotation equivariant models.

\textbf{Main results.}
In our main results, we compare our method, \method{Ours}, to \method{MC} and \method{STD} for different samples taken during training and testing.

\begin{table*}
\caption{Results for different configurations for the classification task on the ModelNet40 dataset. 
The results show that using our sampling approach increases the performance significantly, leading to better results with fewer samples.}
\label{tbl:main-results-modelnet}
\setlength{\tabcolsep}{12pt}
\centering
\small
\begin{tabular}{lcrrrrrrrrr}
    \toprule
    \multicolumn{1}{c}{Method} & \multicolumn{1}{c}{\# samp.} & \multicolumn{3}{c}{I / I}  & \multicolumn{3}{c}{I / SO(3)} & \multicolumn{3}{c}{SO(3) / SO(3)}  \\
    \cmidrule(l{2pt}r{2pt}){3-5}\cmidrule(l{2pt}r{2pt}){6-8}\cmidrule(l{2pt}r{2pt}){9-11}
    & train $\downarrow$/ test $\rightarrow$ & 1 & 2 & 4 & 1 & 2 & 4 & 1 & 2 & 4 \\
    \midrule
    \multirow{3}{*}{\method{MC}}
    & 1 & \cbg 85.4 & \cbg 84.6 & \cbg 83.1 & 78.8 & 74.1 & 70.1 & \cbg 86.5 & \cbg 85.6 & \cbg 84.4  \\
    & 2 & \cbg 86.2 & \cbg 87.0 & \cbg 87.1 & 80.3 & 82.3 & 82.3  & \cbg 87.1 & \cbg 87.0 & \cbg 87.0 \\
    & 4 & \cbg 84.2 & \cbg 87.4 & \cbg 87.5 & 78.4 & 85.6 & 86.2  & \cbg 85.4 & \cbg 88.3 & \cbg 88.2 \\
    \cmidrule{2-11}
    \multirow{3}{*}{\method{Ours}}
    & 1 & \cbg 86.9 & \cbg 86.8 & \cbg 86.7 & 85.5 & 85.3  & 85.3 & \cbg 88.7 & \cbg 88.5 & \cbg 88.5 \\
    & 2 & \cbg 87.9 & \cbg 87.9 & \cbg 87.7 & 86.6 & \textbf{86.9} & 86.8 & \cbg 88.9 & \cbg 88.7 & \cbg 88.7\\
    & 4 & \cbg 73.2 & \cbg 87.6 & \cbg 87.8 & 61.4 & 85.7 & 86.5 & \cbg 59.7 & \cbg \textbf{89.0} & \cbg 88.7 \\
    \midrule
    \method{STD} & & \cbg & \cbg \textbf{90.7} &\cbg & \multicolumn{3}{c}{12.3} & \cbg & \cbg 87.5 & \cbg \\
    \bottomrule
\end{tabular}
\end{table*}

\Cref{tbl:main-results-modelnet} presents the results of this experiment.
As expected, we can see that the standard method \method{STD} achieves good accuracy for I / I.
\method{Ours} and \method{MC}, as it is typical for rotation equivariant networks in this setup, achieve competitive performance but are below \method{STD}.
However, when we look at the more challenging setup, I / SO(3), we can see that \method{Ours} is able to maintain similar accuracy as in the I / I setup, $86.9\,\%$, a drop by only one point in accuracy, while \method{STD} achieves $12.3\,\%$.
\method{MC}, although it can also achieve competitive performance, for most of the cases, the drop in performance is significant compared to the I / I results.
When we look at the SO(3) / SO(3) setup, all three methods achieve good performance; \method{MC} and \method{Ours} are able to outperform \method{STD}, while \method{Ours} achieves the best accuracy. 

Analyzing the effect of different samples used to compute the integral over SO(3) for training and testing, we can see that \method{Ours}, even with 1 sample, can achieve similar results than when using 4 samples.
With only 2 samples, our method is able to match or even surpass the accuracy of using the full frame, 4 samples.
Moreover, using only 1 or 2 samples appears to be more robust than using the full frame, 4 samples, when tested with different numbers of samples. We hypothesize that training with random 1 or 2 samples, rather than using the full frame, introduces stochasticity that acts as a regularizer, enhancing robustness to errors in SO(3) integral estimation. In contrast, \method{MC} is more sensitive to the number of samples, exhibiting significant performance degradation with 1 or 2 samples.

\textbf{Comparison to other methods.}
First, we compare our model to existing non-equivariant point-based network architectures, architectures that rely on global equivariance, and models like ours that use group convolutions to achieve local equivariance.
In ~\cref{tbl:modelnet_sota}, we can see that our model achieves the best performance among the group convolution tested by a large margin in the I/SO(3) setting.
This is due to the discretization of the group SO(3) used by the \method{EPN}~\cite{chen2021equivariant} and \method{E2PN} methods~\cite{zhu2023e2pn}. Also, in the z / SO(3) and SO(3) / SO(3) settings, we outperform all local rotation equivariant networks.
When compared to global equivariant networks, our method falls behind in the I / SO(3) setup and achieves similar performance on the z / SO(3) and SO(3) / SO(3) setup.
However, as we will show later in the segmentation task, while some global equivariant networks only slightly outperform ours on this task, they fail to solve tasks requiring local rotation equivariance. 

\begin{table}[htbp]
\caption{Comparison to equivariant models on the classification task of ModelNet40 for different setups.}
\label{tbl:modelnet_sota}
\setlength{\tabcolsep}{3pt}
\centering
\small
\begin{tabular}{llccc}
    \toprule
    Equiv. & \multicolumn{1}{c}{Method} & \multicolumn{1}{c}{I / SO(3)} & \multicolumn{1}{c}{z / SO(3)} & \multicolumn{1}{c}{SO(3) / SO(3)}\\

    \midrule

    \multirow{5}{*}{\rotatebox[origin=c]{90}{None}}
    & \method{PointNet}~\cite{qi2017pointnet} & -- & \cbg 19.6 & \textbf{84.9} \\
    & \method{PointNet++}~\cite{qi2017plusplus} & 13.8 & \cbg 28.4 & \textbf{84.9} \\
    & \method{DGCNN}~\cite{dgcnn} &  \textbf{17.3} & \cbg 33.8 & 84.8 \\
    & \method{PointCNN}~\cite{YangyanPCNN} & -- & \cbg \textbf{41.2} & 84.8 \\   
    & \method{KPConv}~\cite{thomas2019KPConv} & 12.7 & \cbg  -- & 81.2 \\
    \midrule
    \multirow{5}{*}{\rotatebox[origin=c]{90}{Global}}
    & \method{GC-Conv}~\cite{zhang2020gc} & -- & \cbg 89.1 & 89.2 \\
    & \method{FA-PointNet}~\cite{puny2022frame} & 85.9 & \cbg 85.5  & 85.8 \\
    & \method{FA-DGCNN}~\cite{puny2022frame} & 88.4 & \cbg 88.9 & 88.5\\
    & \method{VN-PointNet}~\cite{deng2021vector} & 77.2 & \cbg 77.5 & 77.2 \\
    & \method{VN-DGCNN}~\cite{deng2021vector} & \textbf{90.0} & \cbg \textbf{89.5} &\textbf{90.2} \\

    \midrule
    \multirow{6}{*}{\rotatebox[origin=c]{90}{Local}}
    & \method{TFN}~\cite{thomas2018tensor} & -- & \cbg 85.3 & 87.6 \\
    & \method{ClusterNet}~\cite{chen2019clusternet} & -- & \cbg 86.4 & 86.4 \\
    & \method{RI-Conv}~\cite{zhang-riconv-3dv19} & -- & \cbg 86.4 & 86.4 \\
    & \method{SPHNet}~\cite{poulenard2019rotinv} & -- & \cbg 86.6 & 87.6 \\
    & \method{EPN}~\cite{chen2021equivariant} &  32.3 & \cbg-- &87.8 \\
    & \method{E2PN}~\cite{zhu2023e2pn} &  44.4 &\cbg -- & 88.6 \\

    \cmidrule{2-5}

    & \method{Ours} & \textbf{86.9} & \cbg \textbf{87.0} & \textbf{89.0} \\
  
    \bottomrule
\end{tabular}
\end{table}

\begin{figure*}[htbp]
    \centering
    \includegraphics*[width=\linewidth]{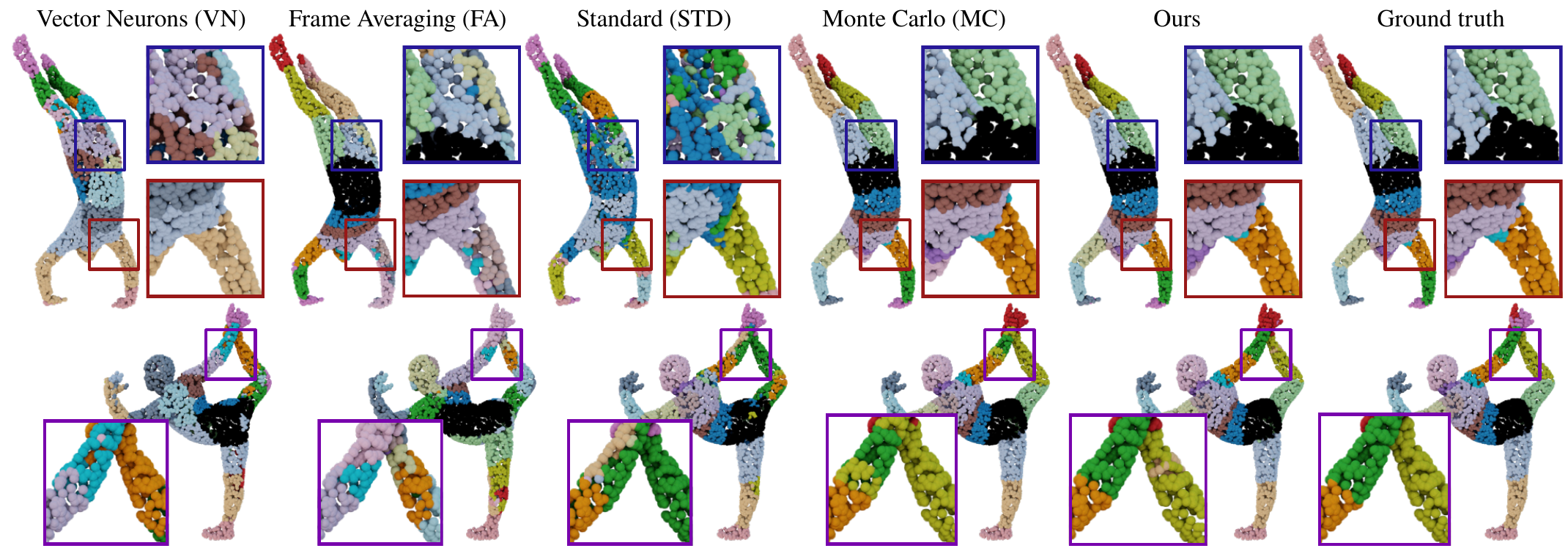}
    \caption{\textbf{Qualitative results.} Global equivariant methods such as \method{VN}, or \method{FA} struggle with out-of-distribution models.
    Our method, on the other hand, achieves almost perfect predictions.
    Lastly, \method{MC} also achieves good performance but falls behind our method, which better approximates the group convolution integral.}
    \label{fig:qualitative}
\end{figure*}

\subsection{Semantic segmentation}
\label{sub-sec:segmentation}

In semantic segmentation, incorporating symmetries like SE(3) equivariance is key for generalization, especially due to the varying orientations and part compositions in point clouds. We evaluate our method on body part segmentation and scene understanding.

\subsubsection{Human body parts}
For semantic segmentation of human body parts, the local equivariance property is essential to distinguish correctly between parts undergoing diverse SE(3) transformations within the kinematic tree. 
Due to the additional symmetry information, we show that our models can generalize to unseen, out-of-distribution poses.

\textbf{Dataset.}
For training and testing, we use two subsets of the AMASS meta-dataset~\cite{mahmood2019amass}, DFAUST~\cite{bogo2017dynamic} and PosePrior~\cite{akhter2015pose}, respectively. 
The PosePrior dataset consists of challenging poses significantly divergent from those executed in DFAUST, which we use to test our model for generalization to unseen, out-of-distribution poses.

\textbf{Experimental setup.}
To assess the ability of our method to generalize to local transformations, we adopt a setup in which we do not use any rotation during training or testing.
Since the testing data is composed of rare poses not seen during training, the models must become invariant to transformations of the different local parts.

\textbf{Main results.}
\Cref{tbl:main-results-dfaust} presents the results of the main experiment.
We can see that \method{STD} struggles to generalize to these out-of-distribution poses, achieving a mAcc of $85.3$ and mIoU of $74.5$.
\method{Ours}, on the other hand, achieves better performance with $95.0$ mAcc and $90.8$ mIoU.
\method{MC} can also achieve competitive performance, but, as in the classification task, this is lower than our proposed approach.

\begin{table}
    \caption{Semantic segmentation results for different models trained on DFAUST and tested on PosePrior. By using our sampling approach, mAcc, and mIoU increase significantly with only a few samples of the frame.}

    \centering
    \small
    \setlength{\tabcolsep}{4pt}
    \label{tbl:main-results-dfaust}
    \begin{tabular}{lcrrrrrr}
        \toprule
        \multicolumn{1}{c}{Method} & \multicolumn{1}{c}{\# samp.} & \multicolumn{3}{c}{mAcc} & \multicolumn{3}{c}{mIoU} \\
        \cmidrule(l{2pt}r{2pt}){3-5}\cmidrule(l{2pt}r{2pt}){6-8}
        & {\small train $\downarrow$/ test $\rightarrow$} &  1 & 2 & 4 & 1 & 2 & 4\\
        \midrule
        \multirow{3}{*}{\method{MC}}
        & 1 & \cbg 93.1 & \cbg93.0 & \cbg92.7 &  87.7 &  87.6 &  87.2 \\
        & 2 & \cbg 93.8 &\cbg 93.9 & \cbg93.8 &  88.8 &  89.0 &  88.7\\
        & 4 & \cbg 93.4 &\cbg 94.2 & \cbg94.4 &  87.9 &  89.3 &  89.7\\
        \cmidrule{2-8}
        \multirow{3}{*}{\method{Ours}}
        & 1 & \cbg 93.8 &\cbg 93.9 & \cbg93.9 &  88.9 &  88.9 &  89.0\\
        & 2 & \cbg 94.3 &\cbg 94.4 & \cbg94.5 &  89.7 &  89.9 &  89.9\\
        & 4 & \cbg 32.6 & \cbg92.4 & \cbg\textbf{95.0} &  21.6 &  86.8 &  \textbf{90.8} \\
        \midrule
        \method{STD} &   &\cbg & \cbg85.3 & \cbg & \multicolumn{3}{c}{74.5} \\
        \bottomrule
    \end{tabular}
\end{table}

\begin{table}
    \caption{Comparison of our method to other rotation equivariant models on the segmentation task for out-of-distribution poses.}
    \centering
    \small
    \setlength{\tabcolsep}{13pt}
    \label{tbl:faust_sota}
    \begin{tabular}{llcc}
        \toprule
        Equiv. & \multicolumn{1}{c}{Method} & \multicolumn{1}{c}{mAcc} & \multicolumn{1}{c}{mIoU}\\
        \midrule
        \multirow{4}{*}{\rotatebox[origin=c]{90}{Global}}
        & \method{FA-PointNet}~\cite{puny2022frame} & \cbg 77.4 & 64.7  \\
        & \method{FA-DGCNN}~\cite{puny2022frame} & \cbg 81.7 & 71.0  \\
        & \method{VN-PointNet}~\cite{deng2021vector} & \cbg  63.1 & 47.5  \\
        & \method{VN-DGCNN}~\cite{deng2021vector} & \cbg  61.1 & 46.6  \\
    
        \midrule
        \multirow{3}{*}{\rotatebox[origin=c]{90}{Local}}
         & \method{EPN}~\cite{chen2021equivariant} & \cbg 89.9 & 82.3\\
         & \method{E2PN}~\cite{zhu2023e2pn} & \cbg  94.8 & 90.7\\
         & \method{Ours} & \cbg  \textbf{95.0} & \textbf{90.8} \\
      
        \bottomrule
    \end{tabular}

\end{table}

When evaluating the model robustness to the number of samples in the SO(3) integral, \method{Ours} outperforms \method{MC} in all cases except when trained on 4 samples but tested on one, as seen in the classification task.

\textbf{Comparison to other methods.}
In Tbl.~\ref{tbl:faust_sota}, we present the results of comparing our method to other global and local equivariant point-based networks.
We can see that \method{Ours} achieves an impressive performance of $95.0$ mAcc and $90.8$ mIoU.
Contrary to the task of shape classification, global equivariant models struggle to generalize to out-of-distribution local transformations not seen during training.
Fig.~\ref{fig:qualitative} depicts predictions for different models tested on the dataset.
The results show that global equivariant methods such as \method{VN} or \method{FA} struggle with out-of-distribution models, confusing legs and arms and right and left. 
The same is true for our non-equivariant version, \method{STD}. 
The training data contain mostly upright positions, i.e., feet are, on average, further down on the z-axis. In contrast, the hands and the head are further up, leading to generalization errors in those models, e.g., the handstand pose as shown in Fig.~\ref{fig:qualitative}.
\method{Ours}, on the other hand, achieves predictions comparable to the ground truth annotations despite never seen those extreme poses during training.
\method{MC} also achieves remarkable performance but performs several prediction mistakes due to inefficient sampling of Frame elements as Tbl.~\ref{tbl:main-results-dfaust} indicates.

When comparing to current state-of-the-art local equivariant methods, we can see that while they also outperform global equivariant methods by a large margin, our method gives superior results, with \method{E2PN}~\cite{zhu2023e2pn} reaching a slightly lower performance.

\begin{table}
\caption{Computational and memory resources of a single convolution layer for our approach and state-of-the-art methods.}
\label{tbl:memory}
\centering
\small
\setlength{\tabcolsep}{12pt}
\begin{tabular}{lcrr}
    \toprule
    \multicolumn{1}{c}{Method} & \# samp. & \multicolumn{1}{c}{Mem. (Mb) $\downarrow$} & \multicolumn{1}{c}{FPS $\uparrow$}\\
    \midrule
    \method{STD} & & 37.1 & 704.2 \\
    \cmidrule{2-4}
    \multirow{3}{*}{\method{Ours}} 
    & 1 &  37.1  & 581.4\\
    & 2 &  76.9  & 432.9\\
    & 4 &  165.2 & 255.8\\
    \cmidrule{2-4}
    \method{E2PN}~\cite{zhu2023e2pn} & & 1211.6 & 45.0 \\
    \method{EPN}~\cite{chen2021equivariant} & & 1636.4 & 10.2 \\
    \bottomrule
\end{tabular}
\end{table}

Tbl.~\ref{tbl:memory} compares a forward pass of a single convolution layer using 1024 points and $256$ input and output features.
We can see that using only one sample to approximate the integral over SO(3) has approximately similar memory consumption and frames per second (FPS) as the non-SO(3) equivariant version of our model. 
This shows that with our method, we can introduce the equivariant property without extra costs, demonstrating the efficiency of our proposed model.
When we analyze the two-sample version of our group convolution, we can see that memory and computation increase by a factor of 2, still making it suitable for its applicability.
When using 4 samples, the memory and computations increase significantly. 
Compared to other state-of-the-art local rotation equivariant methods, \method{E2PN}~\cite{zhu2023e2pn} and \method{EPN}~\cite{chen2021equivariant}, the computational resources needed for our approach are significantly lower even when using 4 samples.

\subsubsection{Scene understanding}

Scenes consist of multiple parts or objects with arbitrary orientations, making local equivariance essential for generalizing to unseen configurations.

\textbf{Dataset.}
We test our method on ScanNet~\cite{dai2017scannet}, a dataset composed of several indoor 3D scene scans, to show its applicability to real-world scenarios.

\textbf{Experimental Setup.}
Since our surroundings have a notion of an up orientation, we fix the z-axis and conduct our experiments for SO(2).
We sample only one orientation from the frame for all experiments, which does not pose additional memory or computational burden on the model. This is a crucial property for processing such large point clouds, making it intractable for the other methods to run reasonable-sized networks for this task. 

\textbf{Main Results.}
Tbl.~\ref{tbl:scannet} shows that our method outperforms \method{STD} in all three configurations, underlining the benefits of baking SE(3) equivariance in the model architecture.
Compared to \method{MC}, we can see that our approach obtains better predictions in all but one configuration.

\begin{table}[t]
\caption{Results for the semantic segmentation task on ScanNet20 show that using our sampling approach increases the performance.}
\label{tbl:scannet}
\setlength{\tabcolsep}{5.5pt}
\centering
\small
\begin{tabular}{lcccccc}
    \toprule
    Method &  \multicolumn{2}{c}{I / I}  & \multicolumn{2}{c}{I / SO(2)} & \multicolumn{2}{c}{SO(2) / SO(2)}  \\

    \cmidrule(l{2pt}r{2pt}){2-3}\cmidrule(l{2pt}r{2pt}){4-5}\cmidrule(l{2pt}r{2pt}){6-7}
    & mAcc & mIoU & mAcc & mIoU & mAcc & mIoU \\
    \midrule
    \method{MC} & 73.4 & 64.5   & \cbg \textbf{74.1} & \cbg 65.2 & 74.2  & 65.7 \\
    \cmidrule(l{2pt}r{2pt}){2-7}
    \method{Ours} & \textbf{73.6} & \textbf{65.6} & \cbg 72.7 & \cbg \textbf{65.4} & \textbf{75.6}  & \textbf{67.5} \\
    \midrule
    \method{STD} & 73.0 & 64.4 & \cbg 70.9 & \cbg 63.5 & 74.5 & 66.4 \\
    \bottomrule
\end{tabular}
\end{table}

\section{Conclusions}
\label{sec:conclusions}

This paper presents an instance of group convolutions on the continuous domain, which is equivariant to SE(3).
Using a carefully constructed subset of group elements makes our operation computationally and memory efficient, obtaining competitive performance when only one sample is taken to solve the integral over the SO(3) group and, therefore, not requiring additional resources over a standard convolution. 
Moreover, by restricting the receptive field of our convolution, our operation becomes local equivariant, allowing us to be robust to local transformations.
Our extensive evaluation presents our approach as a viable solution to incorporate local equivariance in deep network architectures for point clouds without significant additional cost.
\newpage
{
    \small
    \bibliographystyle{ieeenat_fullname}
    \bibliography{main}
}
\clearpage
\setcounter{page}{1}
\maketitlesupplementary
\appendix

\section{Proof of \Cref{theorem}}
$\Phi_{\mathcal{F}}$ is equivaraint to SE$(3)$ iff $\forall g \in \text{SE}(3):  ( \Phi_\mathcal{F} \circ \rho^{\mathcal{X}(g)})(f)= (\rho^{\mathcal{Y}}(g)\circ \Phi_\mathcal{F})(f)$. 
\begin{proof}
Using the definition 
\begin{equation}
    k_{\text{R}}(\text{x, R'}) \coloneq k(\text{R}^{-1}\text{x}, \text{R}^{-1} \text{R'}),
\end{equation}
omitting the normalization constant $\frac{1}{\lvert \mathcal{F}(x_j) \rvert}$ for brevity and using $g=\text{(t}_g, \text{ R}_g\text{)}$ we can write
\begin{equation}
\small
    \begin{split}
        &\big[( \Phi_\mathcal{F} \circ \rho^{\mathcal{X}(g)})(f)\big](\text{x, R})= \\
        &\sum_{j} \sum_{(\text{t, R'})\in \mathcal{F}(x_j)}  f(g^{-1}\text{(t, R')})k(\text{(x, R)}^{-1}\text{(t, R')})=\\
        &\sum_{j} \sum_{(\text{t, R'})\in \mathcal{F}(x_j)}  f(g^{-1}\text{(t, R')})k_{\text{R}}(\text{t} - \text{x},\text{R'})=\\
        &\sum_{j} \sum_{(\text{t, R'})\in \mathcal{F}(x_j)}  f(\text{R}_g^{-1}(\text{t}-\text{t}_g), \text{R}_g^{-1}\text{R'})k_{\text{R}}(\text{t} - \text{x},\text{R'})\overset{x_j \leftarrow \text{R}_g x_j + \text{t}_g}=\\
        &\sum_{j} \sum_{(\text{t, R'})\in \mathcal{F}(\text{R}_g x_j + \text{t}_g)}  f(\text{R}_g^{-1}(\text{t}-\text{t}_g), \text{R}_g^{-1}\text{R'})k_{\text{R}}(\text{t} - \text{x},\text{R'})\overset{\text{equiv. of }\mathcal{F}}=\\
        &\sum_{j} \sum_{(\text{t, R'})\in \text{R}_g\mathcal{F}( x_j )+ \text{t}_g}  f(\text{R}_g^{-1}(\text{t}-\text{t}_g), \text{R}_g^{-1}\text{R'})k_{\text{R}}(\text{t} - \text{x},\text{R'})=\\
        &\sum_{j} \sum_{(\text{R}_g\text{t} +\text{t}_g, \text{R}_g\text{R'})\in \mathcal{F}( x_j )}  f(\text{R}_g^{-1}(\text{t}-\text{t}_g), \text{R}_g^{-1}\text{R'})k_{\text{R}}(\text{t} - \text{x},\text{R'})=\\
        &\sum_{j} \sum_{(\text{t, R'})\in \mathcal{F}( x_j )}  f(\text{t, R'})k_{\text{R}}(\text{R}_g\text{t} +\text{t}_g - \text{x},\text{R}_g\text{R'})=\\
        &\sum_{j} \sum_{(\text{t, R'})\in \mathcal{F}( x_j )}  f(\text{t, R'})k(\text{(x, R)}^{-1} (\text{t}_g, \text{ R}_g)(\text{t}, \text{ R'}))=\\
        &\sum_{j} \sum_{(\text{t, R'})\in \mathcal{F}( x_j )}  f(\text{t, R'})k\big((\text{(t}_g, \text{ R}_g\text{)}^{-1}\text{(x, R)})^{-1} \text{(t, R')}\big)=\\
        &\big[(\Phi_\mathcal{F})(f)\big](g^{-1}\text{(x, R)})= \big[(\rho^{\mathcal{Y}}(g)\circ \Phi_\mathcal{F})(f)\big]\text{(x, R)} \\
    \end{split}
\end{equation}
\end{proof}

\section{Relation to concurrent work}
Previous works have also achieved (piecewise) equivariance by averaging over a frame~\cite{xiao2020pca,puny2022frame}.
While Xiao \etal~\cite{xiao2020pca} only explore this idea to obtain global rotation equivariance, Puny \etal~\cite{puny2022frame} propose a more general framework based on the same idea. The concurrent work of Atzmon~\etal~\cite{atzmonapproximately} uses this idea to achieve piecewise (local) equivariance by applying frame averaging for partitions of individually transforming regions. In these works, the same frame is used for each point in the point cloud (global) or partitions (piecewise/local) for the symmetrization of, e.g., a neural network by transforming their (partitions') domain. 
For Atzmon~\etal~\cite{atzmonapproximately}, a partition prediction network is necessary since simply increasing the number of partitions to reduce partition errors limits the expressivity of the resulting equivariant point network; when each point belongs to one partition, the only shared equivariant function is constant.
In contrast, this work is based on group convolutions, where point convolutions in a neural network are lifted to the SE(3) group so the kernel can detect rotated patterns. In this context, the concept of equivariant frames is used to define a point-specific grid on SE(3) to solve the convolution integral over the SE(3) group efficiently. Since we create point-specific frames, we avoid the need to group points into regions that can rotate jointly. The locality of the features created by the proposed convolution operator is determined by the point neighborhoods used for feature aggregation. Stacking several layers of such convolution operators results in a network capable of detecting local equivariant features up to features equivariant to the accumulated receptive field among the layers. Using efficient SE(3) equivariant convolutions to construct a network, while it can potentially overfit to global features, provides a more general framework than applying symmetrization of the network for constructed regions.

\section{Pose estimation}
\label{sec:posestimation}
Additionally, we evaluate our convolutions in the pose estimation task.
In this task, the model aims to recover the relative rotation between two point clouds of the same shape.

\textbf{Dataset.}
We follow Zhu et al.~\cite{zhu2023e2pn} and use the airplane category from the ModelNet40 dataset for this task, composed of $626$ models in the training set and $100$ in the test set.
As Zhu et al.~\cite{zhu2023e2pn}, we also sample the surface of the shapes with $1024$ points to generate the point clouds, which are randomly rotated to form a pair.

\textbf{Experimental setup.}
Existing methods such as \method{EPN}~\cite{chen2021equivariant} and \method{E2PN}~\cite{zhu2023e2pn} rely on the discretization of the SO(3) group to achieve equivariance.
Therefore, their pose estimation network must predict an assignment between the two discrete sets of rotations plus a displacement to cover the full SO(3) space.
Our convolutions, on the other hand, do not rely on a discrete sampling of the SO(3) space and work directly in continuous space.
Therefore, our models only need to predict an assignment between the $4$ reference frames, which can be framed as a contrastive learning problem. 

\textbf{Main results.}
\Cref{tab:pose-estimation} presents the results of this experiment.
We can see that our model, thanks to operating in continuous space, can achieve an angular error orders of magnitude smaller than the existing methods for any number of samples used in the layers of the network.

\begin{table}[]
    \centering
    \small
    \caption{Error in degrees of different methods on the task of pose estimation on ModelNet40.}
    \label{tbl:pose}
    \setlength{\tabcolsep}{5pt}
    \begin{tabular}{ccccc}
        \toprule
       Metrics & \# samp. & Mean(\degree) & Median(\degree) & Max(\degree)\\
       \midrule
       \method{EPN}~\cite{chen2021equivariant} & & 1.10 & 1.36 & 7.06\\
       \method{E2PN}~\cite{zhu2023e2pn} & & 1.20 & 0.96 & 6.71\\
       \multirow{3}{*}{\method{Ours}}  & 4 & $4.4 \times 10^{-5}$ & $6.7\times 10^{-5}$  & $2\times 10^{-3}$ \\
       & 2 & $4.9 \times 10^{-5}$ & $6.7\times 10^{-5}$  & $2\times 10^{-3}$ \\
       & 1 & $4.9 \times 10^{-5}$ & $6.7\times 10^{-5}$  & $2 \times 10^{-3}$ \\
       \bottomrule
    \end{tabular}
    
    \label{tab:pose-estimation}
\end{table}

\section{Dataset details}

\paragraph{Shape classification.}The ModelNet40 dataset~\cite{wu2015modelnet} is composed of synthetic CAD models from $40$ different classes.
The dataset is divided into two splits, where $9,843$ objects are used for training and $2,468$ for testing.
Since each model is composed of multiple faces, we sample $4,096$ points using farthest point sampling on the surface. 
\paragraph{Semantic segmentation: human body parts.}
For training, we use the train split of DFAUST~\cite{bogo2017dynamic} used in \cite{atzmon2022frame, chen2021snarf} and follow Feng~\etal~\cite{feng2023generalizing} to create $15,430$ point clouds by sampling $4,096$ points across the mesh surface. 
The PosePrior dataset~\cite{akhter2015pose} consists of challenging poses significantly divergent from those executed in DFAUST, which we use to test our model for generalization to unseen, out-of-distribution poses. 
Following the procedure of the train set, we derive $3,760$ point clouds with $4,096$ points each from this dataset for testing.
\paragraph{Semantic segmentation: scene understanding.} 
We follow the standard train and validation split of ScanNet~\cite{dai2017scannet} and use color $[r,g,b]$ as input point features in addition to the 3D coordinates.

\section{Implementation details}
In this section, implementation details are given, and the architecture of the network used is introduced. Classification, pose estimation, and segmentation tasks share the same encoder structure, yet a decoder is used to provide point-wise predictions for the latter. 

\paragraph{Frame computation.}
To compute the local PCA for each point in the point cloud, we select $16$ neighbors using \ac{kNN}.
We compute the covariance matrix from the points and define the frame from the axes of PCA;  \cref{tbl:knn_size} shows an ablation of $k$.

\paragraph{Rotation representation.}
To represent the relative rotations between neighboring points that we give as input to the learnable kernel, we use the 6D representation proposed by Zhou~\etal~\cite{zhou2019continuity}.
However, other viable representations, such as quaternions or rotation matrices, could be used.

\subsection{Network architecture}
For our experiments, our model uses ResNetFormer blocks~\cite{yu2022metaformer} as the main computational blocks in the encoder and an FPN decoder~\cite{kirillov2019pfn} for tasks requiring per-point predictions.
The different point cloud resolutions are computed using Cell Averaging~\cite{thomas2019KPConv}.
In our convolution operation, we define our kernel as a single layer \ac{MLP} with $32$ hidden neurons and GELU activation functions. The output of our network is several feature vectors for each point that correspond to the sampled rotations. We use mean pooling as the projection layer \cref{eq:proj-layer} to get the final output feature per point, but any other pooling can be used. 

\subsubsection{Encoder}
\label{arch:encoder}
The input point cloud is transformed into $n$ down-scaled versions of itself using the Cell Average (CA) method~\cite{thomas2019KPConv}. For the first down-scaling, the size of the voxel cells used in the CA algorithm is a hyper-parameter, $d$, which is then sequentially doubled for each of the following down-scaling steps. The initial features are obtained with a patch encoder similar to the one used in vision transformers~\cite{dosovitskiy2010image}. The patch encoder allows us to extract features from a smaller cell size, which usually increases the model's performance as more points are available while keeping computational costs within limits. We use one additional level with two convolutions for the patch encoder for the classification and segmentation task on DFAUST; for ScanNet20, we skip the patch encoder. The (extracted) initial features are further processed with a set of Metaformer blocks~\cite{yu2022metaformer} before being transferred down to the next down-scaled point cloud via a convolution operation. This procedure is iterated until we reach the final down-scaled point cloud. For pose estimation and classification, we use $n=5$ and mean-aggregation of features in the case of classification. The aggregated feature vector is then passed through a linear layer to perform the final prediction. For the segmentation task on DFAUST and ScanNet20, the features of each level of the encoder serve as input to the decoder, with $n=4$ and $n=5$, respectively. 

\paragraph{Metaformer blocks.} We incorporate the block design defined by Yu \etal~\cite{yu2022metaformer} into our architecture, replacing the attention module of transformers with our point convolution. Each block consists of two residual blocks. In the first one, feature updates are computed using point convolution, while in the second one, updates are determined through a point-wise MLP with two layers. In this MLP, the initial layer doubles the feature count, while the second layer reduces it to the desired output number.

\subsubsection{Decoder}
\label{arch:decoder}
Our Decoder architecture is based on the feature pyramid network proposed by Kirillov~\etal~\cite{kirillov2019pfn}. The input to the decoder is the feature map of the down-scaled point cloud for which a stepwise up-sampling with our point convolutions is employed, progressing from the lowest level to the first down-scaled point cloud. To enhance information and gradient flow, we incorporate skip connections, where features from both the encoder and decoder are summed, producing a distinct feature map for each down-scaled point cloud. Subsequently, each feature map is up-sampled to the initial down-scaled point cloud through a singular up-sampling operation. The resulting $n$ feature maps are then aggregated through summation. If applicable, this feature map is put through a patch decoder, inverting the patch encoder operation. Finally, it is up-sampled to the intended prediction positions by a final convolution and processed by a one-layer MLP to obtain the point-wise predictions.

\section{Experimental setup}

In all experiments, we use AdamW~\cite{loshchilov2018decoupled} as optimizer with a weight decay value of $1^{-4}$ and OneCycleLr~\cite{smith2019super} as learning rate scheduler. Moreover, we employ drop residual paths depending on the depth of the layer~\cite{larsson2017ultra} and gradient clipping for gradient norms exceeding 100. We used label smoothing with a parameter of $0.2$ to prevent overfitting. All models were trained on a single NVIDIA GeForce RTX 3090. The experiment-specific setup is given below.

\paragraph{Classification.}
For the point cloud classification experiments on ModelNet40, we use the encoder architecture explained in~\cref{arch:encoder} with the number of blocks and the number of features for each level equal to [2, 3, 4, 6, 4] and [32, 64, 128, 256, 512], respectively. The initial grid resolution was $d= 0.05$ and a maximum drop rate of $0.2$. All models were trained for 500 epochs using a batch size of 12, with a learning rate of $0.01$ and an initial and final division factor of $100$ and $10000$. We used jitter coordinates, mirroring, and random scale augmentation during training. 

\paragraph{Pose estimation.}
For the pose estimation experiment on the airplane category of ModelNet40, we use the same model as for classification.
Note that the projection layer is omitted for this experiment since equivariance instead of invariance is needed. The model using all four frame elements was trained for 500 epochs with a batch size of 8. Using only 2 or 1 element takes longer for the model to converge; we trained for 2k and 4k epochs with batch sizes of 16 and 64, respectively.

\paragraph{Segmentation.}
For the segmentation task on the DFAUST dataset, we again used the encoder architecture introduced in~\cref{arch:encoder} with two blocks per level and a number of features equal to [32, 64, 128, 256]. We trained all models with a batch size of 32 for 150 epochs. The maximum learning rate was 0.005, with an initial division factor of 10 and 1000 as the final factor. Jitter coordinates are used as augmentation during training; the initial grid resolution was $d= 0.05$, and a maximum drop rate of $0.5$. To get the point-wise prediction, the decoder architecture of \cref{arch:decoder} was employed. 

For the segmentation task on the ScanNet20 dataset, we used the same encoder and decoder architecture as for DFAUSt, described in \cref{arch:encoder} and \cref{arch:decoder}, but used five levels with [2, 3, 4, 6, 4] blocks and 
[64, 128, 192, 256, 320] feature dimensions. The initial grid resolution was $d=0.1$, and all models were trained using 250 batches for 600 epochs using standard augmentations such as jitter coordinates, mirroring, random scaling, elastic distortion, and translation.

\paragraph{Projection layer.} Our proposed method provides features for each sample of the SO3 group; if four samples are used, four features per point result. Hence, we must aggregate those features to get to the final point-wise prediction or before averaging in the classification task to get the point-wise invariant feature vectors. We use mean-pooling over the features corresponding to the same coordinates.

\section{Additional qualitative results}
 \Cref{fig:supp_seq} and \cref{fig:supp_qual} provide additional qualitative results.

\begin{figure*}
    \centering
    \includegraphics*[width=\linewidth]{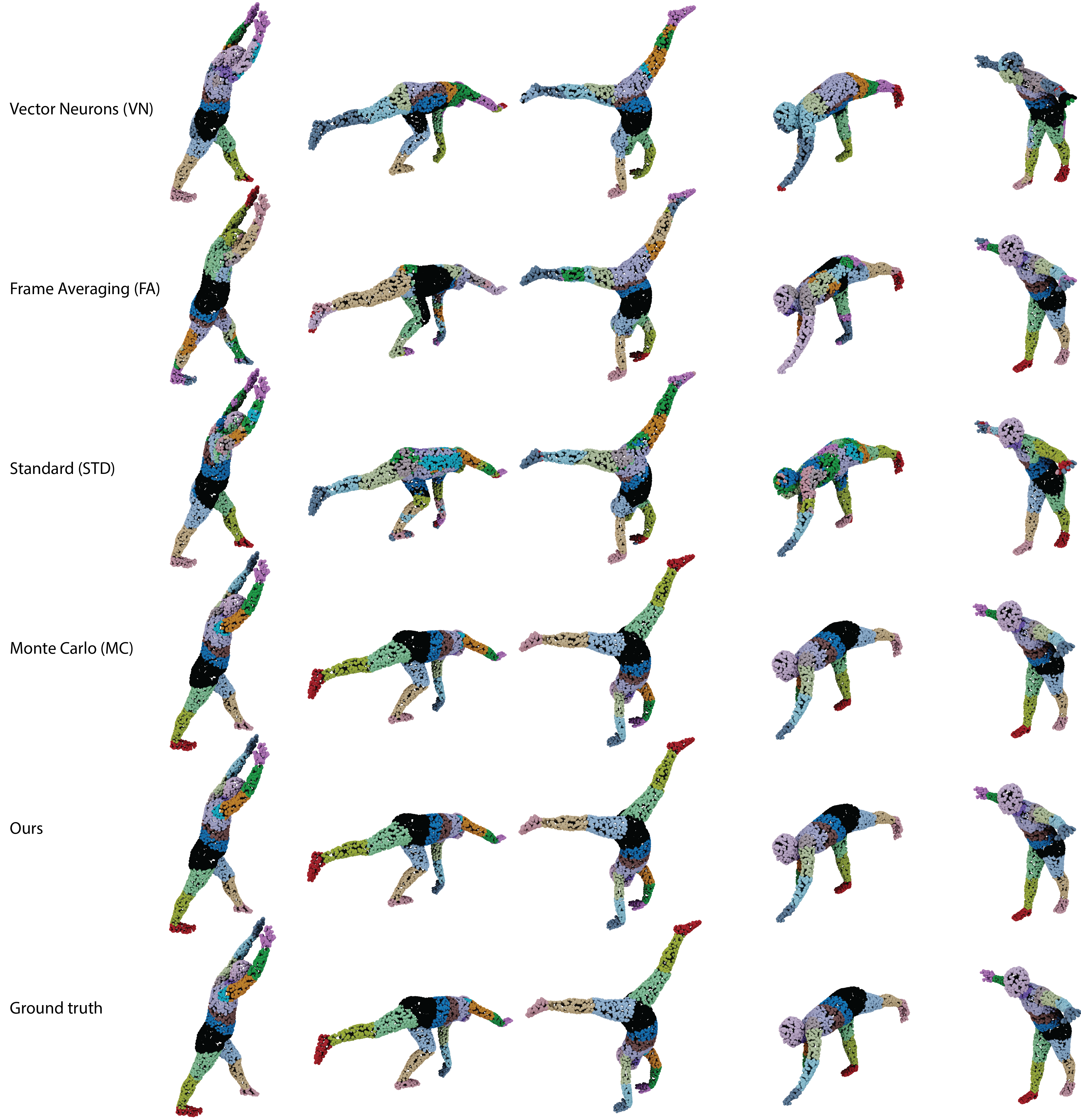}
    \caption{\textbf{Additional Qualitative results.} Global equivariant methods such as \method{VN}, or \method{FA} struggle with out-of-distribution models, especially up-side down models.
    Our method, on the other hand, achieves almost perfect predictions.
    Lastly, \method{MC} also achieves good performance but falls behind our method, as seen in the leftmost columns when looking at the left upper arm prediction.}
    \label{fig:supp_seq}
\end{figure*}

\begin{figure*}
    \centering
    \includegraphics*[width=\linewidth]{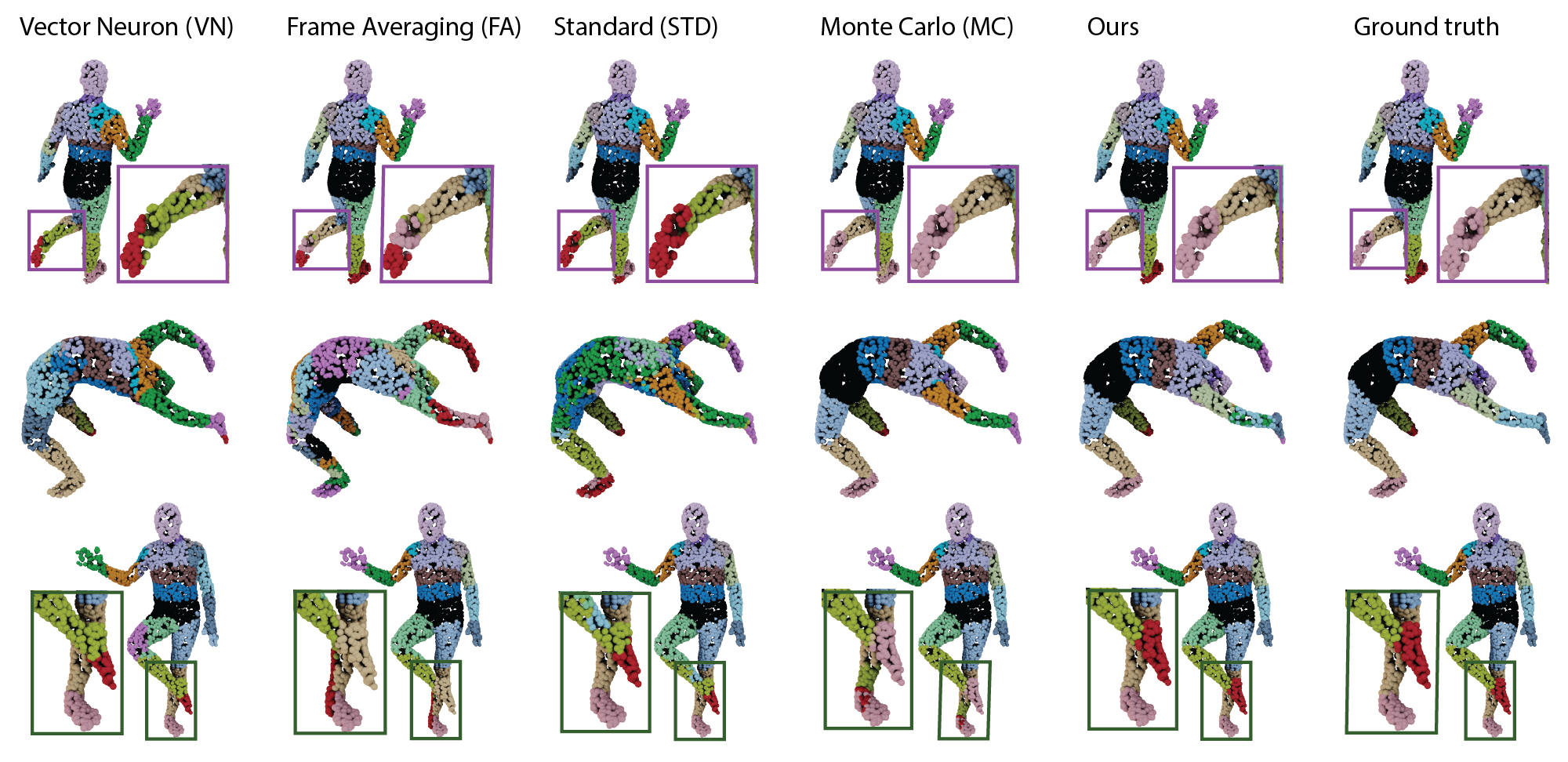}
    \caption{\textbf{Additional Qualitative results.} Global equivariant methods such as \method{VN}, or \method{FA} struggle with out-of-distribution models.
    Our method, on the other hand, achieves almost perfect predictions.
    Lastly, as seen in all three examples, \method{MC} also achieves good performance but falls behind our method.}
    \label{fig:supp_qual}
\end{figure*}

\section{Additional ablations}
In this section, we provide ablation experiments to evaluate the robustness of our approach against the number of neighbors used in the PCA computation and against noise and point density. Further, we provide results for allowing different numbers of samples to be used for the approximation of the SO(3) integral during training, called \emph{sample mixing}. Finally, we discuss the effect of stochastic sampling of the point-specific grid on SE(3).

\paragraph{Sample mixing.}
To achieve the best results with minimal training and inference time, we explored \emph{sample mixing}. Instead of always using the same amount of samples of SO3 elements, the number of samples used changes per step/forward pass and is chosen with the following probability. 1 sample is used with a probability of 50\%, 2 samples with a 35\% and 4 with a 15\% probability. \Cref{tbl:mix-supp-modelnet} and \cref{tbl:mix-supp-dfaust} show the results for the classification and segmentation tasks, respectively. We can see that training with sample mixing and testing with one sample equals or exceeds the performance of training with one sample only for \method{MC} and \method{Ours}. Further, the results are more stable concerning different numbers of samples used during testing than when training with a fixed number of samples. Sample mixing is thus a viable option when limited resources are available. Moreover, training with larger batch sizes becomes feasible by further allowing different numbers of samples not only between but also within batches. \Cref{tbl:time-supp} shows the number of minutes each epoch approximately takes during training with the mixing strategy, 1, 2, and 4 samples. Sample mixing with the proposed probabilities takes, on average, as long as training with two samples while delivering a more robust approximation of the integral.

\begin{table*}
\caption{Results for mixing the number of used samples throughout training for the classification task on the ModelNet40 dataset. 
}
\label{tbl:mix-supp-modelnet}
\setlength{\tabcolsep}{10pt}
\centering
\begin{tabular}{lcrrrrrrrrr}
    \toprule
    \multicolumn{1}{c}{Method} & \multicolumn{1}{c}{\# samp.} & \multicolumn{3}{c}{I / I}  & \multicolumn{3}{c}{I / SO(3)} & \multicolumn{3}{c}{SO(3) / SO(3)} \\
    \cmidrule(l{2pt}r{2pt}){3-5}\cmidrule(l{2pt}r{2pt}){6-8}\cmidrule(l{2pt}r{2pt}){9-11}
    & train $\downarrow$/ test $\rightarrow$ & 1 & 2 & 4 & 1 & 2 & 4 & 1 & 2 & 4 \\
    \midrule
    \method{MC}  & mix & \cbg 84.9 & \cbg 85.7 & \cbg 85.9 & 78.1 & 77.9 & 77.51 & \cbg 86.9 & \cbg 86.8 & \cbg  86.6  \\
    
    \cmidrule{2-11}
    \method{Ours} & mix & \cbg 87.2 & \cbg 86.9 & \cbg 87.0 & 86.3 & 86.2   & 86.3 & \cbg 88.4 & \cbg 88.5 & \cbg 88.4 \\
   
    \midrule
    \method{STD} & & \cbg & \cbg 90.7 &\cbg & \multicolumn{3}{c}{12.3} & \cbg & \cbg 87.5 & \cbg \\
    \bottomrule
\end{tabular}
\end{table*}

\begin{table*}
\caption{Results for mixing the number of used samples throughout training for the segmentation task on the DFAUST dataset. 
}
\label{tbl:mix-supp-dfaust}
\setlength{\tabcolsep}{17.5pt}
\centering
\begin{tabular}{lcrrrrrr}
    \toprule
    \multicolumn{1}{c}{Method} & \multicolumn{1}{c}{\# samp.} & \multicolumn{3}{c}{mAcc} & \multicolumn{3}{c}{mIoU} \\
    \cmidrule(l{2pt}r{2pt}){3-5}\cmidrule(l{2pt}r{2pt}){6-8}
    & train $\downarrow$/ test $\rightarrow$ & 1 & 2 & 4 & 1 & 2 & 4\\
    \midrule
    \method{MC} & mix &  93.8 & 93.7 & 93.7 & \cbg  88.7& \cbg 88.6 & \cbg 88.5 \\
    \cmidrule{2-8}
    \method{Ours} & mix &  94.0 & 94.1 & 94.1 & \cbg 89.2 & \cbg 89.3 & \cbg 89.3\\
    \midrule
    \method{STD} &  & \multicolumn{3}{c}{85.3} & \cbg & \cbg 74.5 &\cbg\\
    \bottomrule
\end{tabular}
\end{table*}

\begin{table}
\caption{Time of 1 epoch in minutes during training with different numbers of samples on the DFAUST dataset.}
\label{tbl:time-supp}
\setlength{\tabcolsep}{15pt}
\centering
\small
\begin{tabular}{rrrrr}
    \toprule
     \# samp. & mix & 1 & 2 & 4 \\
     \cmidrule(l{2pt}r{2pt}){2-5}
      time (min) & 6.1 & 3.6 & 6.1 & 15.0 \\
    \bottomrule
\end{tabular}
\end{table}

\paragraph{PCA computation.}
We also analyze the effect of the receptive field used to compute the PCA for each point on the final performance of the network using one sample to estimate the integral over SO(3).
We can see in \cref{tbl:knn_size} that using a low number of neighboring points makes the resulting frames noisy and hampers the model's performance, becoming similar to the results of \method{MC} using a random grid. 
However, with 16 neighbors, the PCA computation becomes robust, and we do not see significant improvement when we increase this receptive field.

\begin{table}
\caption{Effect of the \emph{k} chosen for the kNN operation in the PCA computation on the model's performance with one frame element on the ModelNet40 dataset.}
\label{tbl:knn_size}
\setlength{\tabcolsep}{16pt}
\centering
\small
\begin{tabular}{ccccc}
    \toprule
    4 & 8 & 16 & 32 & 64\\
    \midrule
    80.6 & 83.1 & 85.5 & 85.4 & 85.9\\
        
    \bottomrule
\end{tabular}

\end{table}

\begin{table}
    \caption{Robustness \wrt noise variations.}
    \centering
    \setlength{\tabcolsep}{19pt}
    \begin{tabular}{cccc}
       \multicolumn{4}{c}{Noise} \\
       \cmidrule(l{2pt}r{2pt}){1-4}
       train & test & mAcc & mIoU \\
       \cmidrule(l{2pt}r{2pt}){1-4}

       \multirow{3}{*}{0.005}
       & 0.005 & \cbg94.8 & 90.6  \\
       & 0.010 & \cbg93.9 & 90.0 \\
       & 0.015 & \cbg38.6 & 25.1 \\
       \cmidrule(l{2pt}r{2pt}){1-4}
       0.015 & 0.015 & \cbg93.2 & 88.2 \\ 
       \bottomrule
    \end{tabular}
    
    \label{tab:abl-robust_noise}
\end{table}
\begin{table}
    \caption{Robustness \wrt density variations.}
    \centering
    \setlength{\tabcolsep}{19pt}
    \begin{tabular}{cccc}
       \multicolumn{4}{c}{Point Density}\\
       \cmidrule(l{2pt}r{2pt}){1-4}
        train & test & mAcc & mIoU\\
       \cmidrule(l{2pt}r{2pt}){1-4}

   \multirow{3}{*}{4096} & 4096 & \cbg94.5 & 89.7 \\
        & 2048 & \cbg93.3 & 87.7 \\
        & 1024 & \cbg30.4 & 20.1 \\
       \cmidrule(l{2pt}r{2pt}){1-4}

        1024 & 1024 & \cbg92.7 & 86.8 \\ 
       \bottomrule
    \end{tabular}
    
    \label{tab:abl-robust_density}
\end{table}

\paragraph{Robustness to noise and density.} We experimented with the robustness of our model \wrt noise and density variations in the input point cloud.
Results on the DFAUST dataset using two samples are reported in \cref{tab:abl-robust_noise} and \cref{tab:abl-robust_density}, respectively.
We can see that our model is robust against increased levels of noise and reduced point density during testing.
However, if the noise increases significantly ($0.015$ std. dev.) or the number of points is reduced substantially ($1024$ points), the PCA computation is affected, and the model's performance decreases.
This can be easily solved by training the model with high noise levels or with a reduced number of points as shown in \cref{tab:abl-robust_noise} and \cref{tab:abl-robust_density}, where the model achieves similar performance to the model trained without point corruptions.

\paragraph{Effects of stochastic sampling}
One of the main contributions of our work is to sample one or two LRF stochastically during training.
How the learning is affected by this sampling boils down to how noisy the gradient estimation for the kernel parameters is.
If this gradient were computed for a single point, the gradient would be noisy.
However, this noise is significantly reduced since the gradient is computed as the expectation over multiple samples, multiple points, and multiple point clouds in the batch.
In our experiments, models trained with 1 LRF or 2 LRF during training perform only marginally worse than those trained with 4.
During testing, sampling 1 or 2 LRF can result in noisy predictions.
However, these predictions remain equivariant since the same sampling of LRF will produce the same results for random SO(3) rotations.

\section{Limitations}
\label{sec:limitations}

Although the proposed convolution operation is local equivariant via the restricted receptive field, when multiple layers are combined in a deep network, the whole model does not become local equivariant and remains global equivariant.
However, from the experiments presented in~\cref{sub-sec:segmentation}, where the network aims to perform local predictions, our model shows robustness to such scenarios, indicating that the model relies on local features to perform the predictions.

\end{document}